%% file: main.tex
\documentclass{article}

\PassOptionsToPackage{dvipsnames}{xcolor}
\usepackage{iclr2027_conference,times}
\input{preprint_mode}

\usepackage[utf8]{inputenc} 
\usepackage[T1]{fontenc}    
\usepackage{amsmath}
\usepackage{amssymb}
\usepackage{mathtools}
\usepackage{amsthm}
\usepackage{enumitem}
\usepackage{algorithm}
\usepackage{algpseudocode}
\usepackage{microtype}
\usepackage{graphicx}
\usepackage{booktabs}
\usepackage[colorlinks,linkcolor={red!80!black},citecolor={blue},allbordercolors={1 1 1},urlcolor={blue!80!black},hypertexnames=false]{hyperref}
\usepackage[capitalize,noabbrev]{cleveref}

\theoremstyle{plain}
\newtheorem{theorem}{Theorem}[section]
\newtheorem{proposition}[theorem]{Proposition}

\theoremstyle{definition}

\theoremstyle{remark}

\usepackage{pgfplots}
\pgfplotsset{compat=1.18}
\usepackage{booktabs,tabularx,ragged2e,xcolor}

\usepackage{tikz}
\usetikzlibrary{positioning,fit,arrows.meta, calc}

\newcommand{\DPVarTrainPenalty}{\gamma}
\newcommand{\DPVarValPenalty}{\zeta}

\algnewcommand{\IIf}[1]{\State\algorithmicif\ #1\ \algorithmicthen}
\algnewcommand{\EndIIf}{\unskip\ \algorithmicend\ \algorithmicif}

\newlist{assumplist}{enumerate}{1}
\setlist[assumplist]{label=(\textbf{\Alph*})}
\Crefname{assumplisti}{Assumption}{Assumptions}

\newlist{assumplist2}{enumerate}{1}
\setlist[assumplist2]{label=(\textbf{\alph*})}
\Crefname{assumplist2i}{Assumption}{Assumptions}

\definecolor{darkred}{rgb}{0.55, 0.0, 0.0}

\definecolor{ink}{HTML}{262626}
\definecolor{accent}{HTML}{0072B2}
\definecolor{accent2}{HTML}{D55E00}

\title{Directly Optimizing Mean Demographic Parity for Nonlinear Regression}

\input{preprint_authors}

\begin{document}
\maketitle

\begin{abstract}
We focus on regression settings where the fairness goal is to equalize average predictions across values of a sensitive attribute, a criterion known as mean demographic parity. Directly optimizing this criterion is difficult because it depends on a conditional mean that is unknown and changes during training. Common dependence penalties and adversarial methods do not estimate this conditional mean; instead, they push predictions toward full independence. This stronger constraint can reduce accuracy even when average predictions are already equal. Existing conditional-mean methods are limited to linear predictors or low-dimensional sensitive attributes. We enable direct optimization of mean demographic parity using DPVar, a fairness measure defined as the variance of the conditional mean prediction. Because the conditional mean must be estimated as the predictor changes, optimizing DPVar leads to a functional bilevel problem. We develop two solvers: FBO, which uses a closed-form hypergradient, and an iterative-differentiation (ITD) solver that differentiates through updates of the conditional-mean estimator. Unlike previous conditional-mean methods, our approach applies to nonlinear predictors and high-dimensional continuous sensitive attributes. Across a semi-synthetic benchmark built from 21 tabular regression datasets and Communities \& Crime data, FBO and ITD recover competitive or better accuracy--DPVar trade-offs than existing methods.
\end{abstract}

\section{Introduction}\label{sec:introduction}

\input{intro.tex}

\section{Related Work}\label{sec:related_work}

\vspace*{-0.2cm}

\input{related.tex}

\section{From DPVar to a Practical Functional Bilevel Algorithm}\label{sec:method}

\vspace*{-0.2cm}

\input{method.tex}

\section{Experiments}\label{sec:experiments}

\vspace*{-0.2cm}

\input{experiments.tex}

\section{Concluding Remarks}\label{sec:conclusion}

\input{ccl.tex}

\section*{AI use statement}
In this work, we used generative AI tools to assist with drafting and editing text, creating or revising scientific visualizations, implementing and reviewing standard experimental code, explaining technical concepts, and formulating or checking mathematical statements. We did not use generative AI tools for translation or for cleaning and reformatting datasets; qualitative and thematic data analysis is not applicable to this work. The authors reviewed all AI-assisted material, checked mathematical claims and citations, tested AI-assisted code, and verified reported results against experiment artifacts. We take responsibility for the final content of this work,
including text, claims, code and figures.

\subsection*{Reproducibility statement}
Section~\ref{sec:method} defines the objective and estimators; Appendices~\ref{app:algorithms} and~\ref{app:hypergrad} give the algorithms, assumptions, and proof. Section~\ref{sec:experiments} and Appendices~\ref{app:controlled-synthetic}--\ref{app:communities-crime} document the data, preprocessing, splits, seeds, budgets, grids, selection, and evaluation. Appendix E specifies the controlled synthetic experiment.

\bibliography{references}
\bibliographystyle{iclr2027_conference}

\newpage
\appendix

\section{Training Algorithms}\label{app:algorithms}
\input{algorithms.tex}
\clearpage

\section{Additional Related Work}\label{app:related_work}

\paragraph{Demographic parity and dependence penalties.}
Demographic-parity methods often regularize training by penalizing dependence between the prediction $f_\omega(X)$ and the sensitive attribute $A$. This includes correlation/covariance penalties (denoted $R^2$) \citep{zafar2017fairness}, kernel-based dependence measures such as HSIC \citep{gretton2008hsic}, and mutual-information inspired objectives \citep{belghazi2018mine,kamishima2012prejudice}. These methods are attractive because they are single-level objectives that can be optimized with standard stochastic-gradient tools. Their limitation, in our setting, is that the optimized dependence penalty is not necessarily the demographic-parity criterion we care about. Depending on the kernels or adversary, it may target full dependence or selected joint moments rather than directly fitting and penalizing $a\mapsto \mathbb{E}[f_\omega(X)\mid A=a]$ \citep{mary2019fairness,jiang2022gdp}.

\paragraph{HSIC baseline and choice of kernel.}
\citet{perezsuay2017fair} propose both a linear formulation (Section~2.2) and a kernel formulation (Section~2.3). For a linear predictor $p=X\beta$ and a sensitive matrix $A\in\mathbb{R}^{n\times k_A}$, their linear solution is
\begin{equation}
    \widehat\beta_{\mu}
    = \left(X^\top X + \lambda I
    + \frac{\mu}{n^2}X^\top HAA^\top HX\right)^{-1}X^\top y,
    \qquad H=I-\frac{1}{n}\mathbf{1}\mathbf{1}^\top .
    \label{eq:appendix-linear-hsic}
\end{equation}
This is HSIC with linear kernels on the scalar prediction and on the complete sensitive vector. It penalizes the squared cross-covariance between $p$ and all columns of $A$ simultaneously. The paper explicitly allows multidimensional sensitive variables. The authors' release confirms this interpretation:\footnote{Author release: \url{https://www.uv.es/pesuaya/data/code/FairKernelLearning2017.zip}.} \texttt{FL.m} forms \texttt{H*Q*Q'*H} and solves exactly the system in \cref{eq:appendix-linear-hsic}, where \texttt{Q} is a matrix rather than a single sensitive column. The accompanying scripts use two indicator columns for sex in \texttt{FKL\_sex.m}, and seven columns for race and sex in \texttt{FKL\_sex\_race.m}. They also report variants that remove these columns from the predictor while retaining the full \texttt{Q} in the fairness penalty, which matches our convention that $A$ is available during training but excluded from $X$ at prediction time.

Our tabular benchmark passes the full continuous $A$ directly to this linear construction. Across the 21 datasets, $k_A$ ranges from 3 to 20; we neither select one coordinate nor project $A$ to a lower-dimensional representation. Because our tasks are regressions, we select the ridge parameter by validation MSE rather than the validation classification accuracy used by the released experiment script. We retain the release's seven-value ridge grid and 25-value fairness path, and add a zero-penalty anchor for comparison with the unconstrained model. Computationally, we use the exact identity
\begin{equation}
    X^\top HAA^\top HX=(A^\top HX)^\top(A^\top HX).
    \label{eq:appendix-linear-hsic-factorization}
\end{equation}
Thus every training observation and every sensitive coordinate is retained, while no $n\times n$ sensitive Gram matrix is materialized. This is an algebraic rearrangement of \cref{eq:appendix-linear-hsic}, not a sampling or kernel approximation.

The paper's nonlinear formulation and the released \texttt{FKL.m} instead replace the linear model in $X$ by RBF kernel-ridge regression and replace the sensitive Gram matrix $AA^\top$ by an RBF Gram matrix on $A$; the HSIC kernel on the scalar fitted predictions remains linear. The RBF kernel on $A$ can detect nonlinear dependence that the linear kernel misses, but \texttt{FKL.m} constructs sample-by-sample input and sensitive Gram matrices and solves an $n\times n$ kernel-ridge system. Our largest tasks have up to 66,666 pooled training observations, for which one dense Gram matrix already contains about $4.4$ billion entries; computing it in blocks would reduce memory but not the quadratic pairwise work or the kernel-system cost. Random features, Nystr\"om approximations, minibatch HSIC, or subsampling could make an RBF penalty feasible, but none is used in the authors' release and each would introduce an additional approximation whose quality must be tuned. We therefore use the authors' exact, scalable linear formulation rather than introduce a new approximate RBF variant. The trade-off is explicit: this baseline controls linear, second-order dependence between predictions and $A$, not arbitrary nonlinear dependence.

\paragraph{Adversarial debiasing.}
Adversarial training enforces invariance by training a predictor jointly with an adversary that attempts to recover $A$ from intermediate representations or from $f_\omega(X)$ \citep{zhang2018mitigating,edwards2016censoring,ganin2016domain}. This approach can target a rich family of dependence notions (depending on the adversary class) and has been successful empirically. At the same time, it can be sensitive to optimization hyperparameters and to the choice of adversary's architecture. Additionally, similarly to dependence penalties, adversarial debiasing does not directly estimate the conditional mean that defines DPVar \citep{jiang2022gdp,mary2019fairness}.

\paragraph{Scope of the CoD comparison.}
The coefficient-of-determination (CoD) method of \citet{pmlr-v80-komiyama18a} handles multiple numeric sensitive attributes and includes nonlinear kernel and random-Fourier-feature extensions, so it is not restricted to linear regression. Its released estimator, however, includes the sensitive attribute in the prediction function and also uses it to residualize the other covariates at validation and test time; the released optimization further depends on Gurobi or Mosek for its QCQP/SDP problems.\footnote{Official implementation: \url{https://github.com/jkomiyama/fairregresion}.} This is a different hypothesis class from our common $f_\omega(X)$ protocol, in which $A$ is removed from $X$ and is unavailable to every predictor at inference. Removing the sensitive-input component or replacing the specialized solver with a neural penalty would therefore be a new adaptation rather than a reproduction. We instead include an $R^2$ baseline that applies the same multiple-linear-predictability statistic as a differentiable penalty to the common neural predictor; we do not present it as the original constrained CoD solver.

\paragraph{Scope of Wasserstein post-processing.}
Fair-regression methods based on Wasserstein barycenters post-process an unconstrained regressor by mapping each finite sensitive group's conditional output distribution to a common barycenter \citep{chzhen2020wasserstein,legouic2020projection,chzhen2022minimax}. Their prediction rule uses group membership at inference; the released implementation of \citet{chzhen2020wasserstein} specializes the construction to two groups.\footnote{Official implementation: \url{https://github.com/lucaoneto/NIPS2020_Fairness}.} With multidimensional continuous $A$, nearly every observed value defines its own group, so this empirical group-quantile map is not defined without first imposing a discretization or clustering rule. Such a rule would change both the method and its full-distribution fairness target, and would again require $A$ at inference. We therefore report these methods as out of scope rather than introducing an unvalidated discretized variant.


\paragraph{Bilevel fairness methods.}
Our use of bilevel optimization differs from recent bilevel approaches to fair learning \citep{roh2021fairbatch,ozdayi2021bifair,yazdani2024fair,tanji2026fairness}. Earlier bilevel fairness methods use the outer problem to adapt training mechanisms such as minibatch composition or sample weights: FairBatch~\citep{roh2021fairbatch} adjusts group-dependent batch-selection parameters to improve standard group-fairness metrics, while BiFair~\citep{ozdayi2021bifair} uses bilevel optimization for fair learning when demographic labels are available only for a small subset of the data. These methods are bilevel in the training procedure, but their inner problems do not estimate the conditional-mean function defining a fairness criterion. FairBiNN~\citep{yazdani2024fair} takes an architecture-centered view: the model is split into accuracy-responsible and fairness-responsible components, and the bilevel structure comes from optimizing these two sets of layers separately. In this sense, FairBiNN is related to our work as a bilevel fairness method, which is why it appears in \cref{fig:dp-dpvar-literature-map}. However, it differs in two key ways: (i) it targets standard group-fairness losses for discrete sensitive groups rather than the conditional-mean DPVar criterion, and (ii) its accuracy-fairness trade-off is controlled through architectural and optimization choices, such as the placement of fairness layers and the relative learning rates of the two players, rather than by a penalty parameter weighting the DPVar criterion. The bilevel Pareto framework of \citet{tanji2026fairness} is also distinct. Their BADR method is motivated by the observation that standard fairness regularization can produce models that are Pareto-inefficient across predefined groups. In contrast to our approach, the bilevel structure in BADR is not induced by an unknown conditional-mean function. Instead, it is used to select a Pareto-efficient fair model among predefined groups. Thus, existing bilevel fairness methods use bilevel optimization to adapt the training procedure, structure the architecture, or select a point on a group-risk Pareto frontier; in our work, the bilevel structure is induced by the DPVar objective itself, through the estimation of $h^\star_\omega(a)=\mathbb E[f_\omega(X)\mid A=a]$. The general functional bilevel and adjoint framework was introduced by \citet{petrulionyte2024functional}. Here we specialize it to DPVar: the squared inner loss gives a closed-form population adjoint, from which we construct the plug-in update studied in this paper.


\section{Total-Gradient Derivation for the Squared-Loss Bilevel Objective}
\label{app:hypergrad}

\begin{proof}[Proof of \cref{prop:closed_form_hypergrad}]
Assume that the \texttt{IN} and \texttt{OUT} laws have the same marginal distribution for $A$, and that the required moments exist so that differentiation and expectation can be exchanged below.
Let $P_{\text{out}}$ be the law on $(X,Y)$ and $P_{A_{\text{out}}}$ the marginal on $A$. The outer objective can be written as:
\begin{align*}
L_{\text{out}}(\omega,h)
&= \underbrace{\mathbb{E}_{(X,Y)\sim P_{\text{out}}}\left[(f_\omega(X)-Y)^2\right]}_{\text{accuracy}}
+ \alpha\underbrace{\mathbb{E}_{A\sim P_{A_\text{out}}}\left[(h(A)-\mu)^2\right]}_{\text{DP variance penalty}}, \\
&\qquad \mu=\mathbb{E}_{A\sim P_{A_\text{out}}}\left[h(A)\right].
\end{align*}

\paragraph{1) Functional gradient in $h$.} We note $h(A)$, $\varphi(A)$, and $\mathbb{E}_{A\sim P_{A_\text{out}}}$ as $h$, $\varphi$, and $\mathbb{E}$ for readability. We work in $L_2(P_{A_{\text{out}}})$ the space of square integrable functions wrt. the measure $P_{A_\text{out}}$ equipped with the inner product $\langle u,v\rangle_{L_2(P_{A_{\text{out}}})} := \mathbb{E}_{A\sim P_{A_{\text{out}}}}[u(A)v(A)]$. For any direction $\varphi \in L_2(P_{A_{\text{out}}})$, consider the perturbation $h + \varepsilon \varphi$ and define:
\begin{align*}
    \mu_\varepsilon := \mathbb{E}[h+\varepsilon\varphi]
    = \mathbb{E}\left[h\right] + \varepsilon \mathbb{E}[\varphi]
    = \mu + \varepsilon \mathbb{E} [\varphi].
\end{align*}
Then consider the perturbed DP variance penalty:
\begin{align*}
    \mathbb{E}\left[(h+\varepsilon\varphi-\mu_\varepsilon)^2\right]
    &= \mathbb{E}\left[(h+\varepsilon\varphi-\left( \mu + \varepsilon \mathbb{E} [\varphi] \right))^2\right] \\
    &= \mathbb{E}\left[(h-\mu+\varepsilon\left(\varphi - \mathbb{E}[\varphi]\right))^2\right] \\
    &= \mathbb{E}\left[(h-\mu)^2\right]
    + 2\varepsilon \mathbb{E}\left[(h-\mu)\left(\varphi - \mathbb{E}[\varphi]\right)\right]
    + \varepsilon^2 \mathbb{E}\left[\left(\varphi - \mathbb{E}[\varphi]\right)^2\right].
\end{align*}
Keeping only the first-order term in $\varepsilon$ gives the differential
\begin{align*}
    \partial_h \mathbb{E}\left[(h-\mu)^2\right] [\varphi]
    = 2 \mathbb{E}\left[(h-\mu)\left(\varphi - \mathbb{E}[\varphi]\right)\right]
    = 2 \mathbb{E}\left[(h-\mu)\varphi\right],
\end{align*}
since $\mathbb{E}[h-\mu] = \mathbb{E}[h] - \mu = \mu - \mu = 0$. Assuming the \texttt{IN} and \texttt{OUT} splits are drawn i.i.d. from the same population, we have $P_{A_{\text{in}}} = P_{A_{\text{out}}}$, so for any $\varphi\in L_2(P_{A_{\text{in}}})$,
\begin{align*}
    \partial_h L_{\text{out}}(\omega,h)[\varphi]
    &= 2\alpha\mathbb{E}_{A\sim P_{A_{\text{out}}}}\left[(h(A)-\mu)\varphi(A)\right] \\
    &= 2\alpha\mathbb{E}_{A\sim P_{A_{\text{in}}}}\ \left[(h(A)-\mu)\varphi(A)\right].
\end{align*}

\paragraph{2) Adjoint.}
The inner objective is
$L_{\text{in}}(\omega,h)=\mathbb{E}_{(X,A)\sim P_{\text{in}}}\big[(f_\omega(X)-h(A))^2\big]$,
whose Hessian w.r.t.\ $h$ is $\partial_{hh}^2 L_{\text{in}}=2I$ on $L_2(P_{A_{\text{in}}})$.
The adjoint $a^\star_\omega\in L_2(P_{A_{\text{in}}})$ solves
\begin{align*}
    \partial_{hh}^2 L_{\text{in}}(\omega,h^\star_\omega) a^\star_\omega
    = - \partial_h L_{\text{out}}(\omega,h^\star_\omega),
\end{align*}
hence
\begin{align*}
    2a^\star_\omega(a) &= - 2\alpha (h^\star_\omega(a)-\mu_\omega)\\
    a^\star_\omega(a) &= -\alpha (h^\star_\omega(a)-\mu_\omega),
    \qquad
    \mu_\omega := \mathbb{E}_{A\sim P_{A_{\text{out}}}}[h^\star_\omega(A)].
\end{align*}

\paragraph{3) Total gradient} By the functional adjoint identity,
\begin{align*}
\nabla_\omega F(\omega)
=\partial_\omega L_{\text{out}}(\omega,h^\star_\omega)
+\partial_{\omega h}^2 L_{\text{in}}(\omega,h^\star_\omega)[a^\star_\omega].
\end{align*}
Compute each term
\begin{itemize}
    \item \textbf{Direct outer term:}
      \begin{align*}
          \partial_\omega L_{\text{out}}(\omega,h^\star_\omega)
          =2 \mathbb{E}_{\text{out}}\left[(f_\omega(X)-Y) \partial_\omega f_\omega(X)\right].
      \end{align*}
    \item \textbf{Implicit (adjoint) term:}
      \begin{align*}
          \partial_h L_{\text{in}}(\omega,h)[\varphi]
          = -2 \mathbb{E}_{\text{in}}\left[(f_\omega(X)-h(A)) \varphi(A)\right],
      \end{align*}
      differentiating w.r.t.\ $\omega$:
      \begin{align*}
          \partial_{\omega h}^2 L_{\text{in}}(\omega,h)[a]
          = -2 \mathbb{E}_{\text{in}}\left[a(A) \partial_\omega f_\omega(X)\right],
      \end{align*}
      writing the full implicit term:
      \begin{align*}
          \partial_{\omega h}^2 L_{\text{in}}(\omega,h^\star_\omega)[a^\star_\omega]
          =-2 \mathbb{E}_{\text{in}}\left[a^\star_\omega(A) \partial_\omega f_\omega(X)\right]
          =2\alpha \mathbb{E}_{\text{in}}\left[\big(h^\star_\omega(A)-\mu_\omega\big) \partial_\omega f_\omega(X)\right].
      \end{align*}
\end{itemize}

Putting it together,
\begin{align*}
    \nabla_\omega F(\omega)
    = 2 \mathbb{E}_{X,Y\sim P_{\text{out}}}\left[(f_\omega(X)-Y)\partial_\omega f_\omega(X)\right]
     + 
    2\alpha \mathbb{E}_{A,X\sim P_{\text{in}}}\left[\big(h^\star_\omega(A)-\mu_\omega\big) \partial_\omega f_\omega(X)\right],
    \end{align*}
    \quad
    \begin{align*}
    \mu_\omega=\mathbb{E}_{A\sim P_{A_{\text{out}}}}[h^\star_\omega(A)].
\end{align*}
\end{proof}

\section{Example: Job Recommendation.}\label{app:job_example}
Consider job recommendation in online advertising, where some demographic profiles should not systematically receive lower-salary recommendations on average. This is the setting considered in the FairJob benchmark \citep{vladimirova2024fairjob}. In this benchmark, the protected attribute $A$ is not directly observed; instead, the released dataset provides a binary gender proxy inferred from users' behavioral interactions. This binary proxy makes the dataset compatible with standard demographic-parity methods, which are usually designed for discrete sensitive groups and do not directly apply to high-dimensional continuous sensitive attributes. At the same time, it highlights a limitation of this approach: a richer behavioral signal must be reduced to a hard demographic label before these methods can be used. Our method avoids this reduction when the richer sensitive representation is available. Rather than collapsing the sensitive signal to a binary class and enforcing independence from that class, it can work directly with a richer representation and control whether predicted salaries or recommendation scores are systematically higher or lower on average across demographic profiles. Moreover, because our approach controls the conditional expectation rather than enforcing full independence, it does not require the entire distribution of recommendations to be identical across all demographic profiles. This is important in job recommendation, where different groups may legitimately differ in preferences or constraints, such as desired flexibility, remote-work options, commuting tolerance, or schedule compatibility. Enforcing full independence risks washing out this structure and producing recommendations that are less useful to everyone. By contrast, controlling the conditional expectation means that demographic profiles receive comparable salary levels on average, while still allowing recommendations to reflect heterogeneous needs.


\section{Controlled Synthetic Experiment}\label{app:controlled-synthetic}

Let $U,V,W,Z,\xi,\varepsilon\stackrel{\mathrm{iid}}{\sim}\mathcal N(0,1)$ and define
\begin{align*}
A&=\left(\frac{U+\kappa W}{\sqrt{1+\kappa^2}},
          \frac{U-\kappa W}{\sqrt{1+\kappa^2}},V\right),
&q(W)&=\frac{\sin(2W)-2e^{-2}W}{c_q},\\
X_1&=q(W)+\sigma_x\xi,
&X_2&=(1+\rho U)Z,
&Y&=q(W)+\beta X_2+\sigma_y\varepsilon,
\end{align*}
where $c_q^2=(1-e^{-8})/2-4e^{-4}$ and
$(\kappa,\rho,\beta,\sigma_x,\sigma_y)=(.35,.5,2,.15,.1)$. The first two coordinates of $A$ determine $U$ and $W$. The normalization gives $\operatorname{Var}(q(W))=1$, while the subtraction makes $q(W)$ uncorrelated with $W$. Thus $X_1$ contains a nonlinear mean disparity that a linear dependence penalty can miss. In contrast, $\mathbb E[X_2\mid A]=0$, although $\operatorname{Var}(X_2\mid A)=(1+\rho U)^2$ changes with $A$; penalizing that change is unnecessary for mean parity.

For the linear reference $f_b(X)=b_1X_1+b_2X_2$,
\begin{align*}
\mathbb E[f_b(X)\mid A]&=b_1q(W), &
\operatorname{DPVar}(f_b)&=b_1^2,\\
\operatorname{MSE}(f_b)&=(b_1-1)^2+\sigma_x^2b_1^2
 +(1+\rho^2)(b_2-\beta)^2+\sigma_y^2.
\end{align*}
Consequently, minimizing $\operatorname{MSE}+\alpha\operatorname{DPVar}$ gives
$b_1=(1+\sigma_x^2+\alpha)^{-1}$ and $b_2=\beta$. This curve is the dashed population reference in \cref{fig:controlled-synthetic}.

Each seed draws four independent splits of 2,000 observations. Every method trains one model at zero penalty and 16 logarithmically spaced positive penalties: $[10^{-3},10^2]$ for FBO and adversarial debiasing, $[10^{-2},10^3]$ for ITD, $[10^{-4},10^2]$ for GDP\(^{\ast}\) and $R^2$, and $[10^{-2},10^4]$ for HSIC. For the simulator-based test metric, we hold each of the 2,000 test values of $A$ fixed, resample $(Z,\xi)$ 128 times, average the resulting predictions, and take the variance across $A$, subtracting the finite-draw Monte Carlo contribution. The same draws are shared across methods within a seed. All 17 trained models are reported; no test quantity is used for selection.

\section{Tabular Benchmark Protocol}\label{app:tabular-protocol}

\paragraph{Task roster and feature types.}
We begin with a versioned roster of 105 OpenML and UCI regression tasks. The prepared artifacts record the source dataset and target, feature names and positions, semantic types, missing values, and checksums. The benchmark is numeric-only: categorical, Boolean, date/time, and identifier-like columns are excluded. Only source-continuous columns can be selected for $A$; discrete numeric columns may remain in $X$.

\paragraph{Separating construction from evaluation.}
We use stable source-row identifiers and a deterministic hash to divide each task into a construction pool and a disjoint experiment pool. The construction pool targets 20\% of the complete rows, with a minimum of 2,000 and a cap of 20,000, and is used for every target-informed decision described below. It is split again into \texttt{FIT}, \texttt{DIAGNOSTIC}, and \texttt{EVALUATION} subsets in proportions $1/2$, $1/4$, and $1/4$. No construction row is used to train, validate, or test a fairness method. After withholding this pool, a task must retain at least 5,000 complete experiment rows.

\paragraph{Constructing $A$.}
On construction \texttt{FIT}, we first identify continuous, nonconstant features with sufficient numerical variation. Candidate features are screened by their absolute correlation with the target. We then score each candidate using both its target correlation and a proxy score measuring its correlation with other available features. This favors coordinates that matter for prediction but whose information is not isolated to a single input column. The highest-scoring quarter forms $A$, with a minimum of three coordinates; at least one coordinate must remain in $X$. The selected coordinates are removed from $X$ in every subsequent experiment.

This construction is designed to make the benchmark informative rather than artificially easy. If $A$ were unrelated to the target, enforcing fairness would have little effect on accuracy. If its information disappeared when its columns were removed, a predictor using only $X$ would have little disparity to mitigate. Selecting target-associated coordinates with proxies in the remaining features creates a measurable accuracy--fairness trade-off across datasets. The selected coordinates are synthetic sensitive attributes for this stress test; we do not claim that they are sensitive in their source applications.

\paragraph{Dataset eligibility.}
We check the construction on held-out rows before any fairness method is trained. A task is retained only if (i) a reference predictor using $X$ achieves $R^2\geq .10$, (ii) $A$ explains at least $10\%$ of the variance in its predictions, and (iii) removing the estimated $A$-dependent component increases MSE by at least $2\%$ of the target variance. These checks rule out tasks with no useful predictive signal, no measurable disparity, or no meaningful accuracy--fairness tension. Structural rules also remove duplicates and tasks requiring temporal, grouped, or survival-specific evaluation. Applying the fixed rules to the 105-task roster yields the 21 benchmark datasets. This set is frozen before the fairness methods are run; validation ranges and test results cannot change benchmark membership.

\paragraph{Final experiment splits and preprocessing.}
For each final run, we use at most 100,000 rows from the experiment pool and allocate them in proportions $1/3$, $1/3$, $1/6$, and $1/6$ to \texttt{IN}, \texttt{OUT}, \texttt{VAL}, and \texttt{TEST}. Centering and scaling are fitted on \texttt{IN}+\texttt{OUT} and applied unchanged to \texttt{VAL} and \texttt{TEST}. All experiment-pool rows are complete, so the defensive imputation code does not alter the data. The benchmark manifest records the selected feature names and indices, row-partition rule, and construction cutoff.

\paragraph{Hyperparameter-search budgets.}
The gradient-trained methods use matched search budgets. FBO, ITD, $R^2$, adversarial debiasing, and GDP\(^{\ast}\) each evaluate eight method-specific hyperparameter settings for every positive penalty. All five methods start from
$\{0,.5,2,10,50,200,1000,5000\}$ and use the same validation-driven refinement rule, with at most twelve distinct penalties. At zero penalty, inactive fairness hyperparameters are collapsed to two learning-rate choices; hence each method evaluates at most $2+11\times8=90$ configurations. Every gradient-trained configuration uses 300 predictor updates and the same cross-fitted validation DPVar diagnostic.

HSIC follows a different, method-specific protocol because it has a released closed-form solution rather than a gradient-trained neural objective. We select its ridge once by validation MSE from the released seven-value grid, freeze that choice, and evaluate the released 25-value fairness path together with a zero-penalty anchor. Thus the gradient-trained methods receive the same maximum number of tuning opportunities, while HSIC follows its published search path. Per-configuration runtimes are not matched because the methods have different computational costs.

\section{Communities \& Crime Experiment}
\label{app:communities-crime}

\paragraph{Data and sensitive profile.}
The Communities \& Crime dataset combines community-level demographic and socioeconomic measurements with reported violent crime for 1,994 US communities \citep{redmond2002communities}. We predict the standardized \texttt{ViolentCrimesPerPop} outcome. The target-independent sensitive profile is
\begin{align*}
A = (&\texttt{racepctblack},\texttt{racePctWhite},\texttt{racePctAsian},\texttt{racePctHisp},\\
&\texttt{agePct12t21},\texttt{agePct12t29},\texttt{agePct16t24},\texttt{agePct65up},\\
&\texttt{PctForeignBorn},\texttt{PctNotSpeakEnglWell})\in\mathbb{R}^{10}.
\end{align*}
We remove these coordinates, identifiers, race-conditioned income and police fields, immigration/language-composition fields, and the highly incomplete police block from the predictor input, leaving the same 72 predictors in every run. Missing-value imputation and standardization are fit on \texttt{IN}+\texttt{OUT}. Each target-stratified split contains 600 \texttt{IN}, 600 \texttt{OUT}, 397 \texttt{VAL}, and 397 \texttt{TEST} communities.

\paragraph{Training protocol.}
Every method trains exactly 32 distinct models per seed. For FBO, ITD, $R^2$, adversarial debiasing, and GDP\(^{\ast}\), we select 32 configurations from a 13-value penalty grid and the method-specific settings used in the tabular benchmark. This subset is fixed before any validation metric is observed, covers every penalty value, and retains both predictor learning rates at zero penalty. Each configuration uses 500 predictor updates. HSIC follows the search procedure in its released implementation rather than the neural-method grid: seven ridge candidates are compared by validation MSE, the selected zero-penalty model is reused, and 25 additional models are fitted along the fairness path. This also gives exactly 32 fits. The number of fitted models is therefore identical across methods.

Validation DPVar is estimated with five folds, three shared restarts, and 300 fitting steps. Only the union of configurations chosen by validation scalarizations and the validation Pareto set is evaluated on \texttt{TEST}, where the same five-fold, three-restart diagnostic is refit for 1,000 steps.

\begin{table}[ht]
\centering
\small
\caption{Mean paired difference $\Delta_w=(\mathrm{MSE}+w\,\mathrm{DPVar})_m-(\mathrm{MSE}+w\,\mathrm{DPVar})_{\mathrm{GDP}^{\ast}}$ on Communities \& Crime. Negative values favor the DPVar-targeting method; parentheses give the number of five seeds with a negative difference. Configurations are selected on validation only.}
\label{tab:communities-crime-paired}
\begin{tabular}{lrrrr}
\toprule
Method & $w=0$ & $w=.5$ & $w=2$ & $w=10$ \\
\midrule
ITD--GDP\(^{\ast}\) & $+.011$ (1/5) & $-.023$ (4/5) & $-.124$ (5/5) & $+.641$ (0/5) \\
FBO--GDP\(^{\ast}\) & $+.010$ (2/5) & $-.036$ (4/5) & $-.143$ (5/5) & $-.025$ (4/5) \\
\bottomrule
\end{tabular}
\end{table}

\begin{figure*}[ht]
  \centering
  \includegraphics[width=0.99\linewidth]{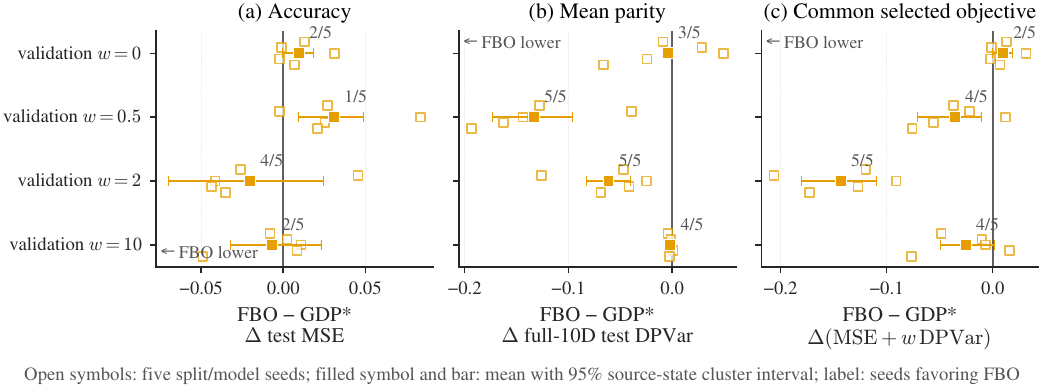}
  \caption{\textbf{Test paired FBO--GDP\(^{\ast}\) comparison.} Each row uses the configuration minimizing the corresponding validation objective before test evaluation. Open squares are the five split/model seeds; the filled square and bar give the mean and pointwise 95\% source-state-cluster interval conditional on the five fitted pipelines. Negative values favor FBO. Panel (c) uses the same weight $w$ as validation selection. The labels count seeds favoring FBO.}
  \label{fig:communities-crime-paired}
\end{figure*}

At $w=2$, FBO lowers mean MSE by $.0202$ and DPVar by $.0616$, lowering the common test objective by $.1434$ in all five seeds; its pointwise conditional interval is $[-.181,-.110]$. At $w=10$, the corresponding changes are $-.0071$, $-.0018$, and $-.0254$ in four seeds, with interval $[-.049,.002]$. These intervals resample the 44 observed source-state clusters with shared draws across methods and seeds. Because the five test splits reuse one finite dataset, they are not five independent datasets; even five same-sign pairs give an exact two-sided sign-test $p=.0625$. We therefore interpret the pattern as consistency across fitted pipelines, not as a confirmatory significance claim.

As an estimator check, an independently cross-fitted ExtraTrees regressor yields DPVar values closely aligned with the neural diagnostic across 60 distinct selected evaluations (Pearson $r=.999$, Spearman $\rho=.995$). Median training time over five seeds is $.58$ seconds for the full HSIC search, $9.40$ seconds for FBO, $46.57$ seconds for ITD, and $111.32$ seconds for GDP\(^{\ast}\). At $w=10$, ITD selects its maximum penalty in all five seeds, FBO in three, and GDP\(^{\ast}\) in four, so the strong-weight comparisons should be read with this boundary diagnostic in mind.

\end{document}

%% file: preprint_mode.tex
\usepackage{etoolbox}

\makeatletter
\@ifundefined{iclrfinalcopy}{%
  \PackageError{preprint_mode}{ICLR final-copy command not found}{%
    Load iclr2027_conference before preprint_mode.tex.}%
}{%
  \iclrfinalcopy
}

\patchcmd{\@maketitle}
  {Published as a conference paper at ICLR 2027}
  {Preprint}
  {}
  {\PackageError{preprint_mode}{Could not patch the ICLR 2027 title header}{%
    Check that the project uses the official iclr2027_conference.sty.}}
\makeatother

%% file: preprint_authors.tex
\author{\noindent Ieva Petrulionyte, Julien Mairal, Michael Arbel \\
Univ. Grenoble Alpes, Inria, CNRS, Grenoble INP, LJK, 38000 Grenoble, France \\
\texttt{firstname.lastname@inria.fr}}

\hypersetup{
  pdftitle={Directly Optimizing Mean Demographic Parity for Nonlinear Regression},
  pdfauthor={Ieva Petrulionyte, Julien Mairal, Michael Arbel},
  pdfsubject={Machine learning; algorithmic fairness; nonlinear regression},
  pdfkeywords={demographic parity, fair regression, bilevel optimization, conditional mean}
}

%% file: intro.tex
\par Predictive machine learning models are effective because they exploit statistical signals that reduce prediction error. Unfortunately standard training does not distinguish desirable signals from those associated with attributes deemed ``sensitive'' for legal, ethical, or policy reasons \citep{barocas2016bigdata,mehrabi2021survey}. Because models are often required to be ``fair'' with respect to these attributes, one must choose a fairness criterion and optimize its trade-off with predictive quality. 
Fairness encompasses a \emph{family} of criteria \citep{mehrabi2021survey}: demographic parity (DP) compares predictions across sensitive groups \citep{jiang2022gdp}; equalized odds and equal opportunity condition on the true outcome \citep{hardt2016equality}; and individual and counterfactual fairness require consistency across similar individuals or hypothetical interventions \citep{kusner2017counterfactual,barocas2023fairness}. These criteria are neither equivalent nor always jointly attainable \citep{chouldechova2017fair,kleinberg2017inherent}. We focus on the demographic parity family, motivated by settings where the goal is to control systematic differences in predictions across sensitive groups.

\par A typical problem considered in the DP literature aims at learning a map $f_{\omega}$ that predicts a target $Y$ from non-sensitive features $X$ from supervised data, while enforcing a DP condition w.r.t. to sensitive attribute $A$. In its \emph{strong} form, DP requires statistical independence between the predictor $f_{\omega}(X)$ and sensitive attribute $A$ \citep{kamishima2012prejudice,gretton2008hsic,zhang2018mitigating,edwards2016censoring}, thereby equalizing the entire conditional distribution of predictions across values of $A$. 
If the fairness concern is a systematic shift in \emph{average} prediction, strong DP also constrains spread, tails, and shape, potentially degrading predictive performance. 
In credit scoring and job recommendation \citep{vladimirova2024fairjob}, for example, one may seek comparable average credit scores or salary recommendations across demographic profiles without requiring identical distributions (see~\cref{app:job_example}). 
\emph{Mean demographic parity} relaxes strong DP by requiring only that $\mathbb{E}[f_{\omega}(X)\mid A]$ be approximately constant \citep{pmlr-v206-wei23a,jiang2022gdp,yazdani2024fair}. 
This condition can be naturally encouraged by a penalty term controlling the variance of the conditional mean across $A$, which we call \emph{DPVar} and that was previously considered in the literature in restricted settings \citep{pmlr-v80-komiyama18a,pmlr-v206-wei23a}. As Figure~\ref{fig:dpvar_conceptual} illustrates, DPVar measures how average predictions vary with $A$ and equals zero exactly when mean demographic parity holds.

\begin{figure}
\centering
\begin{tikzpicture}[x=1.0cm,y=1.0cm,line cap=round,line join=round,>=Latex]
  \draw[ink, line width=0.8pt, ->] (0,0) -- (5.1,0) node[right, font=\small] {$a$};
  \draw[ink, line width=0.8pt, ->] (0,0) -- (0,2.6) node[above, font=\small] {$\mathbb{E}[f_\omega(X)\mid A=a]$};

    \draw[accent2, line width=1.5pt, dashed, smooth] plot coordinates {(0.25,0.85) (0.9,1.25) (1.7,1.65) (2.5,1.10) (3.3,1.70) (4.2,0.95) (4.8,1.20)};
    \node[anchor=west, text=accent2, font=\small] at (2.15,1.92) {\textcolor{accent2}{higher DPVar}};

    \draw[accent, line width=1.5pt, smooth] plot coordinates {(0.25,1.13) (0.9,1.3) (1.7,1.29) (2.5,1.10) (2.8,1.13) (3.3,1.29) (4.2,1.15) (4.8,1.21)};
    \node[anchor=west, text=accent, font=\small] at (2.15,0.74) {\textcolor{accent}{lower DPVar}};
\end{tikzpicture}
\caption{Conceptual illustration of DPVar. A predictor with higher DPVar has a conditional-mean
profile $a\mapsto \mathbb E[f_\omega(X)\mid A=a]$ that varies more strongly with the sensitive
attribute. Lower DPVar corresponds to a flatter conditional-mean profile.}
\label{fig:dpvar_conceptual}
\end{figure}
\par DPVar and related fairness measures depend on the conditional mean of predictions, which is generally unavailable in closed form and must be estimated. For discrete $A$, groupwise sample means suffice \citep{charpentier2023mitigating,glasserman2024should,pmlr-v206-wei23a}; for continuous, multidimensional $A$, estimation is more challenging. \citet{pmlr-v80-komiyama18a} address continuous sensitive attributes through a coefficient-of-determination (CoD) constraint and derive tractable algorithms for linear and kernelized least-squares predictors. However, their formulation does not directly accommodate neural networks, which can model more complex relationships. \citet{jiang2022gdp} instead penalize the average absolute deviation of the conditional mean from the overall mean, estimating it through kernel smoothing and numerical integration. This accommodates continuous sensitive attributes and general predictors, including neural networks. In higher dimensions, however, kernel estimation does not scale well to high dimensions, which limits the applicability of such approaches to small dimensional sensitive attributes $A$ (only up to 2 dimensions were considered in practice).

\par In this work, we address the above limitations by showing that direct optimization of DPVar has a natural functional-bilevel \citep{petrulionyte2024functional} structure: the inner problem estimates the predictor-dependent conditional mean, while the outer problem optimizes prediction accuracy and mean parity jointly. The inner problem estimates the conditional mean $h^\star_\omega(a)=\mathbb E[f_\omega(X)\mid A=a]$ by regressing predictions $f_\omega(X)$ on $A$ in function space. Such estimate is provided to the outer problem which minimizes the  predictive loss penalized by DPVar, while accounting for the dependence of $h^\star_\omega$ on the predictor parameters $\omega$. This approach accommodates continuous, multidimensional sensitive attributes and permits neural networks to represent both the predictor and its conditional mean. It also avoids kernel smoothing, whose estimation accuracy can deteriorate as the dimension of $A$ grows. We develop scalable first-order algorithms for the resulting bilevel problem and evaluate them against state-of-the-art baselines on datasets of increasing difficulty. 
Specifically, our contributions are as follows:
\begin{enumerate}[leftmargin=*, itemsep=1pt, topsep=0pt, parsep=0pt]
\item We turn direct DPVar minimization into a tractable learning problem for nonlinear predictors and continuous, high-dimensional sensitive attributes by estimating the required conditional mean as an inner regression function.
\item For squared-loss inner regression, we derive an explicit population hypergradient whose functional Hessian reduces to \(2I\), yielding a scalable plug-in FBO update without parameter-space Hessian inversion; we additionally evaluate ITD as a complementary unrolled solver.  
\item We introduce a semi-synthetic benchmark dataset of 21 public regression tasks, selected to exhibit a nontrivial sensitive-attribute component of the prediction.
\item We demonstrate competitive performance against state-of-the-art baselines, on synthetic tasks, the semi-synthetic benchmark and a real dataset from a Communities \& Crime study using observed ten-dimensional continuous demographic attribute. 
\end{enumerate}
The remainder of the paper proceeds as follows. Section~\ref{sec:related_work} reviews related work. Section~\ref{sec:method} introduces the functional bilevel problem and the proposed algorithms. Section~\ref{sec:experiments} reports experiments.

%% file: related.tex
\textbf{Demographic parity and mean parity.}
Most demographic-parity methods seek to make predictions independent of the sensitive attribute. This requirement has been approached through constraints and reductions \citep{zafar2017fairness,agarwal2018reductions}, dependence penalties based on kernels or mutual information \citep{perezsuay2017fair,li2022kernel,kamishima2012prejudice}, and adversarial training \citep{edwards2016censoring,zhang2018mitigating}. For continuous sensitive attributes, methods based on Rényi's maximal correlation criterion likewise target dependence between predictions and $A$ \citep{mary2019fairness,grari2020fairness}. Optimal-transport methods pursue the same distributional notion of parity by mapping group-conditional prediction distributions toward a Wasserstein barycenter \citep{chzhen2020wasserstein,legouic2020projection,chzhen2022minimax}. Related barycenter methods address other performance measures \citep{gaucher2023wasserstein}.
Our objective concerns a weaker requirement: the conditional \emph{mean} of a real-valued prediction should be constant across values of $A$. Mean parity leaves other features of the prediction distribution unconstrained. It has been studied directly in RKHS regression for discrete sensitive attributes \citep{pmlr-v206-wei23a}; we encourage it by penalizing DPVar and propose a method that handle continuous high-dimensional attributes.

\textbf{Continuous sensitive attributes and conditional-mean criteria.}
Work closest to our setting handles continuous sensitive attributes by exploiting the structure of the predictor, smoothing over values of $A$, or partitioning $A$ into groups. \citet{pmlr-v80-komiyama18a} constrain the proportion of prediction variance explained by multiple numeric sensitive attributes, using optimization tailored to linear or kernelized least-squares predictors, but which does not directly extend to neural networks. \citet{jiang2022gdp} 
measure the average absolute difference between the conditional prediction mean and the overall mean estimated through kernel smoothing and numerical integration over sensitive values, and thus does not scale well in high dimensions.  
Finally, \citet{shilova2025fairness} identify partitions of continuous sensitive attributes that reveal disparities between the resulting groups. Their partitioning algorithms are developed for one- and two-dimensional attributes, while scaling the search to higher dimensions remains an open challenge. In this work, we measure conditional-mean variation with respect to the original, ungrouped $A$ and estimate that mean through a learned regression function, allowing neural predictors without discretizing $A$ or applying kernel smoothing over it.

\textbf{Bilevel optimization for fairness.}
Bilevel optimization has been used in fair learning to adjust training weights, batch composition, and other choices in the learning procedure \citep{roh2021fairbatch,ozdayi2021bifair,yazdani2024fair,tanji2026fairness}. For DPVar, the nested structure arises instead because evaluating the fairness criterion requires a conditional mean, which can itself be defined as the solution of a regression problem. This connects our setting to functional bilevel optimization, where the inner solution is a function \citep{petrulionyte2024functional}.

Figure~\ref{fig:dp-dpvar-literature-map} summarizes the relationships among these lines of work; further detailed discussions of the various methods appear in \cref{app:related_work}.

\input{figure_related.tex}

%% file: figure_related.tex
\begin{figure*}[ht]
\centering
\begin{tikzpicture}[
  font=\footnotesize,
  >=Stealth,
  box/.style={
    draw, rounded corners, align=center,
    inner sep=4pt, line width=0.45pt, fill=white
  },
  paper/.style={
    box, text width=32mm
  },
  core/.style={
    box, text width=37mm
  },
  arr/.style={-Stealth, line width=0.45pt},
  darr/.style={-Stealth, dashed, line width=0.45pt},
  lab/.style={
    midway, fill=white, inner sep=1pt, font=\scriptsize, align=center
  }
]

\node[core] (dp) {\textbf{Demographic Parity (DP)}\\$f(X)\perp A$};

\node[core, below=10mm of dp] (dpv)
{\textbf{DPVar}\\
$\text{Var}_A\left(\mathbb{E}\left[f(X)\mid A\right]\right)$\\
conditional-mean / weak-DP view};

\node[paper, left=14mm of dp, yshift=0mm] (directdp)
{\textbf{Direct DP methods}\\ constraints, reductions};

\node[paper, below=2mm of directdp] (proxydp)
{\textbf{Approximate DP methods}\\ HSIC, MI, adversarial};

\node[paper, right=14mm of dp, yshift=0mm] (ot)
{\textbf{OT / Wasserstein}\\post-processing,\\ full distribution view};

\node[paper, below=2mm of ot] (conta)
{\textbf{Continuous-$A$ methods}\\ HGR, grouping, IPM};

\node[paper, below=2mm of conta] (gdp)
{\textbf{Conditional-mean DP summaries}\\GDP, CoD,\\ absolute-deviation view};

\node[paper, below=4mm of dpv, text width=39mm, fill=gray!18] (thiswork)
{\textbf{This work: FBO / ITD}\\ bilevel optimization of DPVar};

\node[paper, below=2mm of proxydp] (fairbinn)
{\textbf{Bilevel fairness methods}\\
training mechanisms,\\ architectures
};

\definecolor{axisblue}{RGB}{16,89,166}      

\draw[arr, draw=axisblue] (dp) to[bend left=12]
  node[lab, right, xshift=0mm, text=axisblue] {\tiny DP $\Rightarrow$ DPVar=0} (dpv);

\draw[arr, draw=axisblue] (dpv) to[bend left=12]
  node[lab, left, xshift=-0mm, text=axisblue] {\tiny DPVar=0 $\Rightarrow$ DP\\for binary-valued\\predictions} (dp);

\draw[arr] (directdp.east) -- (dp.west);

\draw[darr] (proxydp.east) -- (dp.west);

\draw[darr] (ot.west) -- (dp.east);

\draw[darr] (conta.west) -- node[lab, above, yshift=+1.4mm] {\scriptsize\shortstack{extensions\\/ penalties}} (dp.east);

\draw[darr] (gdp.west) -- node[lab, above, yshift=+1mm] {\scriptsize\shortstack{shared\\cond.~mean\\view}} (dpv.east);

\draw[arr] (thiswork.north) -- (dpv.south);

\draw[darr] (fairbinn.east) -- node[lab, above, yshift=+1mm] {\scriptsize\shortstack{related\\ bilevel\\view}} (thiswork.west);

\draw[darr] (fairbinn.north) to[out=0,in=250] (dp.south west);

\end{tikzpicture}
\caption{Positioning of DPVar relative to demographic parity (DP) and existing fairness methods. The blue arrows show the relation between full distributional DP and conditional-mean DP. Solid method arrows indicate direct optimization of the target object; dashed arrows indicate related relaxations, penalties, or approximate formulations.}
\label{fig:dp-dpvar-literature-map}
\end{figure*}

%% file: method.tex
\subsection{{Problem formulation}}

\textbf{Setup and objective.}
We consider the learning task of predicting a target variable $Y\in \mathbb{R}$ from non-sensitive features $X\in\mathbb R^{k_X}$, while encouraging mean demographic parity w.r.t. a sensitive attribute $A\in\mathbb R^{k_A}$. Predictions are made using a model $f_\omega: \mathbb R^{k_X}\to\mathbb R$, with parameter vector $\omega\in \Omega$ belonging to some Euclidean subset $\Omega$ and to be learned from data. 
Here, data comes in the form of triplets $(X,Y,A)$ drawn from a probability distribution $P$, with $P_A$ denoting the marginal law of $A$. 
Our learning problem, in population, consists in minimizing an objective $\omega\mapsto  F(\omega)$ that balances mean squared prediction error against our proposed criterion, DPVar, for measuring violations of mean demographic parity by the variation of the conditional mean prediction across $A$
\begin{equation}
  F(\omega)=\mathbb E_P\!\left[(f_\omega(X)-Y)^2\right]
  +\alpha\,\operatorname{DPVar}(\omega),\qquad \operatorname{DPVar}(\omega)
  =\operatorname{Var}_{A\sim P_A}
  \!\left(\mathbb E_P[f_\omega(X)\mid A]\right),
  \label{eq:outer_dpvar}
\end{equation}
where $\alpha$ is a non-negative hyper-parameter that sets the accuracy--mean-parity trade-off. The difficulty with the above problem is that DPVar requires the unknown function $a\mapsto\mathbb E_P[f_\omega(X)\mid A=a]$, which also changes as $\omega$ changes. 

\textbf{Functional bilevel formulation.}
We propose to address this difficulty by  estimating the conditional mean through a regression of the current predictions on $A$. this allows a flexible model to learn shared structure across values of $A$, potentially resulting in efficient estimation. To distinguish the two roles of the data in training, we let \texttt{IN} and \texttt{OUT} be independent samples from the population law $P$. We then write $P_{\mathrm{in}}$ and $P_{\mathrm{out}}$ for their population laws and keep $P_A$ for their common marginal law of $A$. 
The inner sample fits the conditional mean; the outer sample evaluates the prediction loss and the fairness penalty. At the population level, the function being fitted is 
\begin{equation}
  {h^\star_\omega(a)
  =\mathbb E_{(X,A)\sim P_{\mathrm{in}}}[f_\omega(X)\mid A=a].}
  \label{eq:hstar_def}
\end{equation}
{Assuming $f_\omega(X)$ is square-integrable, and setting $\mathcal H=L^2(P_A)$, the least-squares projection property of conditional expectation gives the following equivalent characterization, with uniqueness understood up to $P_A$-almost-sure equality:}
\begin{equation}
  {h^\star_\omega\in\arg\min_{h\in\mathcal H}
  \mathbb E_{(X,A)\sim P_{\mathrm{in}}}
  \!\left[(f_\omega(X)-h(A))^2\right].}
  \label{eq:inner_sq}
\end{equation}
This characterization turns estimation of a conditional expectation into ordinary regression on the \texttt{IN} sample. Substituting the regression solution into the penalized objective thus yields:
\[
  {\min_\omega\;
    \mathbb E_{P_{\mathrm{out}}}[(f_\omega(X)-Y)^2]
    +\alpha\operatorname{Var}_{A\sim P_A}(h^\star_\omega(A)).}
\]
This above problem is a bilevel optimization problem, since the objective depends on $h^\star_\omega$, the solution of the inner regression problem in \cref{eq:inner_sq},  which, in turn, depends implicitly on the parameter $\omega$. 
It is \emph{functional} because the population inner variable belongs to the function space $\mathcal H$, even though, in practice, we will approximate $h^\star_\omega$ with a neural network with finite-dimensional parameter during training \citep{petrulionyte2024functional}. 

\subsection{{Gradient-based algorithms}}
As we will consider predictors that are neural networks, we are interested in optimizing their parameters $\omega$ using gradient methods. To this end, one must estimate $\nabla_\omega F(\omega)$, often called the \emph{hypergradient} in the context of bilevel optimization. It must account for the response of $h^\star_\omega$ to a change in $\omega$. 
We consider two ways to obtain a usable gradient, both of which require estimating $h^\star_\omega$ using an additional model of parameters $\phi$ and which, by abuse of notation, will be denoted as $h_{\phi}$.  \emph{Functional bilevel optimization} (FBO) derives a population hypergradient in the function space $\mathcal H$ and substitutes the estimate $h_{\phi}$ to   $h^\star_\omega$. \emph{Iterative differentiation} (ITD) differentiates through a finite sequence of updates to the parameters $\phi$ of the model $h_{\phi}$. 

\textbf{{Functional bilevel optimization (FBO).}}
When fitting the model $h_{\phi}$ using the regression objective in \cref{eq:inner_sq}, the optimal $\phi$ depends implicitly on the outer parameter $\omega$. Thus one might consider using methods such as approximate implicit-differentiation (AID) which would require a inverting the Hessian of the inner-loss wrt the $\phi$. 
When using an overparameterized neural network,  such a Hessian need not be invertible or well conditioned, so that route can be sensitive to approximation and regularization choices. 
The functional formulation does not require differentiating w.r.t. $\phi$ to estimate the sensitivity of the optimal regressor $h^{\star}_{\omega}$ to $\omega$ \citep{petrulionyte2024functional}. Instead, it needs to compute the Hessian of the inner objective w.r.t. to the function $h$ in the function space $\mathcal{H}$. For the squared loss considered in \cref{eq:inner_sq}, the Hessian with respect to the \emph{function} $h\in L^2(P_A)$ is $2I$, which allows deriving the following simple expression for the hyper-gradient. 
\begin{proposition}[Population gradient for squared-loss DPVar]\label{prop:closed_form_hypergrad}
{Suppose $f_\omega(X)$, $Y$, and the relevant derivatives are square-integrable; and differentiation can be interchanged with the expectations and conditional expectation below. Let $h^\star_\omega$ solve \cref{eq:inner_sq} and define $\mu_\omega=\mathbb E_{A\sim P_A}[h^\star_\omega(A)]$. Then the gradient of \cref{eq:outer_dpvar} can be written as}
\[
{\nabla_\omega F(\omega)
=2\,\mathbb{E}_{\mathrm{out}}\!\left[(f_\omega(X)-Y)\,\partial_\omega f_\omega(X)\right]
+2\alpha\,\mathbb{E}_{\mathrm{in}}\!\left[(h^\star_\omega(A)-\mu_\omega)\,\partial_\omega f_\omega(X)\right].}
\]
\end{proposition}
The proof is in Appendix~\ref{app:hypergrad}. The second expectation contains the dependence of the optimal conditional mean on $\omega$, even though it does not explicitly differentiate the fitted regressor $h_\phi$. FBO substitutes $h_\phi$ for $h^\star_\omega$, estimates $\mu_\omega$ from \texttt{OUT}, and averages the fairness term over \texttt{IN}; it also estimates the accuracy term on \texttt{OUT}. During this plug-in update, $h_\phi$ and the estimated mean are held fixed, as prescribed by the derived formula. The resulting object is a gradient estimate that can become more accurate by increasing the quality of the fit $h_{\phi}$ and sample size.

\textbf{{Iterative differentiation (ITD).}}
{Instead of using the above closed-form identity for the hyper-gradient, ITD starts from the current $\phi$, takes $K$ additional gradient steps on the empirical inner regression loss, evaluates the accuracy and DPVar terms on \texttt{OUT} using the resulting regressor, and differentiates that outer objective through those $K$ steps \citep{pedregosa2016hyper}. This follows how the fitted conditional mean changes with $\omega$ over the unrolled computation. The finite $K$ introduces truncation error and increases time and memory costs, but the procedure still showed good performance, in practice.}

\textbf{{Training loop.}}
{Both methods first fit $h_\phi$ for $M$ steps on \texttt{IN}, then update $\omega$ using the chosen hypergradient estimate:  
ITD additionally unrolls $K$ inner steps for its outer update; FBO instead uses the plug-in identity in \cref{prop:closed_form_hypergrad}. Algorithm~\ref{alg:bilevel-train} shows the shared loop. The inner fitting and two hypergradient computations appear in Algorithms~\ref{alg:inner-fit} and~\ref{alg:hypergradient} in the Appendix.}

\begin{algorithm}[!ht]
\caption{Bilevel Training for DPVar}
\label{alg:bilevel-train}
\begin{algorithmic}[1]
\State \textbf{Input:} initial predictor parameters $\omega_0$; initial inner parameters $\phi_0$; fairness weight $\alpha$; outer stepsize $\eta_{\mathrm{out}}$; inner stepsize $\eta_{\mathrm{in}}$; number of outer steps $N$; number of inner-fit steps $M$; hypergradient mode \textsc{HG} $\in \{\textsc{FBO},\textsc{ITD}\}$; unrolling length $K$
\State $\omega \leftarrow \omega_0$, $\phi \leftarrow \phi_0$
\For{$n=0,\dots,N-1$}
    \If{$\alpha > 0$}
        \State $\phi \leftarrow \textsc{EstimateConditionalMean}(\omega,\phi,\eta_{\mathrm{in}},M,\texttt{IN})$
    \EndIf
    \State $g \leftarrow \textsc{ComputeHypergradient}(\omega,\phi,\alpha,\textsc{HG},K,\eta_{\mathrm{in}},\texttt{IN},\texttt{OUT})$
    \State $\omega \leftarrow \omega - \eta_{\mathrm{out}} g$
\EndFor
\State \textbf{Return:} $\omega$
\end{algorithmic}
\end{algorithm}

%% file: experiments.tex
We evaluate methods on the trade-off between (i) predictive mean-squared error $\mathbb{E}[(f_\omega(X)-Y)^2]$ and (ii) the demographic-parity variance $\text{DPVar}$. Unless stated otherwise, predictors and diagnostic regressors are two-hidden-layer $\tanh$ MLPs (width $64$).

\textbf{Training protocol and evaluation.}
All methods use the same data partitioning and model-selection protocol. Training algorithms use only the \texttt{IN} and \texttt{OUT} splits; we select hyperparameters, including the fairness weight and method-specific optimization parameters, on a separate \texttt{VAL} split; and we report final results on \texttt{TEST}. To evaluate demographic-parity style deviations for continuous $A$ in a way that is comparable across methods, we use the same diagnostic pipeline throughout: for every trained predictor $f_\omega$, we compute DPVar by cross-fitting $h_\phi(A)\approx \mathbb{E}[f_\omega(X)\mid A]$ within the evaluation split (\texttt{VAL} for hyperparameter tuning, \texttt{TEST} for final reporting), fitting on one half and evaluating on the other and averaging across the two roles unless stated otherwise. We use $\alpha$ for the population-level fairness weight in the derivation; in experiments, we distinguish the training penalty weight $\DPVarTrainPenalty$, used to train a model, from the validation weight $\DPVarValPenalty$, used only for model selection along the retained validation trade-off set. Because FBO and ITD must estimate the predictor-dependent conditional mean as an inner problem, \texttt{IN} is reserved for fitting $h_\phi$, while only \texttt{OUT} labels enter the predictor's accuracy loss. The single-level baselines require no such inner fit and therefore optimize prediction accuracy on pooled \texttt{IN}+\texttt{OUT} data. Thus, FBO and ITD operate with fewer labeled examples for accuracy optimization, placing them at a deliberate disadvantage in the predictive comparison.

\textbf{Baselines.}
We compare our two bilevel methods, FBO and ITD, to standard single-level alternatives trained on \texttt{IN}+\texttt{OUT} data:
\begin{itemize}
    \item a penalty based on the CoD statistic of \citet{pmlr-v80-komiyama18a}, which penalizes linear predictability of $f_\omega(X)$ from $A$\footnote{CoD: \url{https://github.com/jkomiyama/fairregresion}.};
    \item GDP\(^{\ast}\), a high-dimensional adaptation of the KDE-based Generalized Demographic Parity method of \citet{jiang2022gdp};
    \item an adversarial-debiasing baseline in which an adversary predicts $A$ from $f_\omega(X)$ \citep{edwards2016censoring,zhang2018mitigating};
    \item an HSIC penalty between predictions and $A$ that follows the released closed-form method of \citet{perezsuay2017fair}.
\end{itemize}
GDP requires special treatment in our setting because its KDE estimator does not scale to high-dimensional sensitive attributes. GDP estimates the mean prediction at each sensitive-attribute value and measures how far these local means are, on average, from the overall mean prediction. Its KDE estimator approximates this integral on a Cartesian grid, whose size grows exponentially with the number of sensitive dimensions. A direct extension of the released estimator\footnote{GDP: \url{https://github.com/zhimengj0326/GDP}.}, using its 18 grid values per coordinate, would require $18^{10}=3.57\times10^{12}$ evaluation points for our ten-dimensional Communities \& Crime attribute and approximately $28.6$PB merely to materialize the coordinate tensor for a batch of $200$. We therefore report GDP$^{\ast}$, which projects $A$ onto two training-fitted principal components before applying the two-dimensional GDP estimator in both experiments.

\textbf{Computational setup.}
All experiments were run on a heterogeneous GPU cluster with NVIDIA GPUs ranging from older GTX/Titan/P100-class cards to RTX A5000/A6000 and RTX 6000 Ada cards. Running the full method and hyperparameter-selection protocol for one tabular dataset takes roughly 2 hours. Within-run timing comparisons for Communities \& Crime are reported in \cref{app:communities-crime}.

\subsection{Controlled experiment with known DPVar}\label{sec:controlled-synthetic}
On real data, the true conditional mean is unknown, so it is difficult to tell whether a penalty reduces the intended mean disparity or some other dependence on $A$. We therefore use a controlled problem whose conditional mean and DPVar can be computed from the simulator. One input feature changes in conditional mean with $A$, while the other changes only in conditional variance; mean parity should penalize the first, not the second. We use 2,000 samples in each of \texttt{IN}, \texttt{OUT}, \texttt{VAL}, and \texttt{TEST} across ten seeds. Every method receives exactly 17 predeclared penalty values, including zero. Iterative methods use 500 predictor updates per fit; HSIC uses its closed-form solve. All fitted models are evaluated once on \texttt{TEST}. Test DPVar is estimated directly from the simulator by Monte Carlo. All details are given in \cref{app:controlled-synthetic}.

\begin{figure}[ht]
  \centering
  \includegraphics[width=0.68\linewidth]{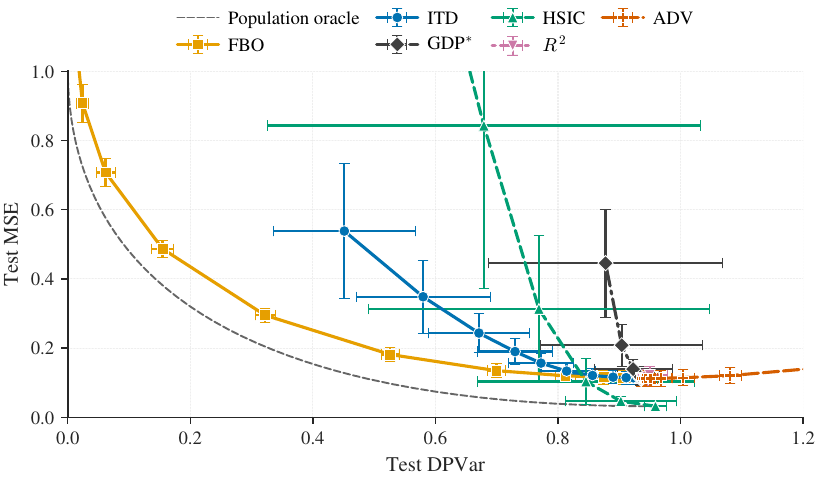}
  \caption{\textbf{Controlled accuracy--DPVar trade-off.} Points are means over ten seeds for 17 fixed penalty values per method; horizontal whiskers are 95\% CI for mean test DPVar; vertical whiskers are 95\% CI for mean test MSE. The dashed curve is the optimal population trade-off within the linear predictor class. Lower-left is better.}
  \label{fig:controlled-synthetic}
\end{figure}

\textbf{Results.}
Figure~\ref{fig:controlled-synthetic} separates three behaviors. FBO follows the population reference over almost the full DPVar range: it learns the nonlinear conditional mean while retaining the variance-only feature for prediction. ITD moves in the same direction but covers a shorter range under the same update budget. HSIC and $R^2$ cannot see the constructed mean signal through linear association; GDP\(^{\ast}\) and adversarial debiasing likewise do not trace the targeted frontier. This controlled experiment checks the underlying theory, while the following benchmarks provide the broader evidence.

\subsection{Semisynthetic benchmark on tabular datasets}\label{sec:tabular}

We define a semi-synthetic tabular benchmark from an explicit list of regression tasks obtained from \emph{OpenML} and \emph{UCI} \citep{dua2019uci,vanschoren2014openml}. The datasets in the benchmark are intentionally numeric-only: categorical, Boolean, date/time, and identifier-like source columns are excluded; only source-continuous coordinates can enter $A$. The benchmark is not meant to claim that these coordinates are legally or socially sensitive in the original tasks. It is a controlled stress test for continuous, multidimensional sensitive attributes with nontrivial proxy information in the remaining predictors.

Benchmark membership is frozen before any fairness method is trained. Fixed construction-time rules require at least 5,000 completely finite feature/target rows in the final experiment pool, and reject classification-like, temporal or extremely sparse tasks. A held-out construction-only screen then requires a useful reference predictor, non-flat reference predictions over $A$, and a meaningful prediction cost when the estimated $A$-dependent component is removed. Applying these rules to the complete 105-task roster yields the 21 benchmark datasets, all fixed before training; no dataset is retained or discarded because of a trained method's validation range or test behavior.

\textbf{Sensitive attributes.}
We construct $A$ using a separate subset that is never used to train or evaluate the final models. Among continuous features, we select those that are related to the target and have information reflected in other features. This makes the benchmark nontrivial: $A$ matters for prediction, and simply removing it from $X$ does not eliminate its influence because proxy information remains. Applying the same fixed selection procedure across many different datasets provides a consistent and robust test of the accuracy--fairness trade-off. Full details are provided in Appendix~\ref{app:tabular-protocol}.

\textbf{Training and model selection.}
For each method, we train a bounded grid of method-specific configurations and select models using only \texttt{VAL}. For each predefined trade-off weight $\DPVarValPenalty$, we retain the configuration minimizing
\[
\mathrm{MSE}_{\mathrm{val}}
+\DPVarValPenalty\,\mathrm{DPVar}_{\mathrm{val}}.
\]
Only these validation-selected configurations are eligible for the final test comparison. This gives all methods comparable tuning budgets and prevents uninformative extreme configurations from distorting the pareto front range. Hyperparameter grids, refinement rules, training budgets, and stopping criteria are detailed in the ~\ref{app:tabular-protocol}.

\textbf{Constrained comparison.}
The tabular benchmark campaign uses 11 seeds (0 through 10) across 21 datasets, for 231 total runs. For each dataset, we pool the validation DPVar values. If $d_{\min}$ and $d_{\max}$ are the extremes of this validation-only pool, the \emph{Tight}, \emph{Mid}, and \emph{Loose} budgets are respectively $d_{\min}+.25(d_{\max}-d_{\min})$, $d_{\min}+.50(d_{\max}-d_{\min})$, and $d_{\min}+.75(d_{\max}-d_{\min})$. For each method, seed, and budget, we freeze the retained configuration with the lowest validation MSE among those satisfying the validation DPVar budget. Test metrics do not enter the budget, feasibility, or configuration choices and are used only after all three are frozen.

For each dataset--seed pair, winner credit is assigned to the validation-feasible method with the lowest frozen test MSE. If no method reaches a budget for a dataset--seed pair, that pair is reported as unattainable and omitted from that seed's win-rate denominator. The figure uses the mean of the 11 seed-wise win rates as the bar height, overlays all 11 seed values as individual points, and uses whiskers for their observed minimum--maximum range.

\begin{figure*}[h!]
\centering
\includegraphics[width=\textwidth]{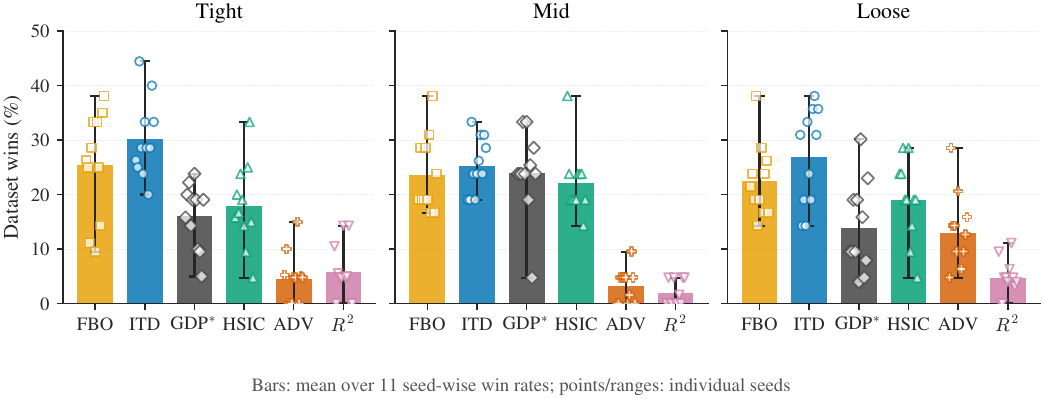}
\caption{\textbf{Win rates on the tabular benchmark.} Tight, Mid, and Loose denote increasingly permissive validation-DPVar thresholds; Tight requires greater mean parity, while Loose allows more disparity. A win is awarded to the feasible method with the lowest test MSE, so higher is better. Bars are means over the 11 seed-wise, equal-dataset-weight win rates; points show seeds 0 through 10; whiskers show the observed minimum and maximum.}
\label{fig:constrained-wins-validation-frozen}
\end{figure*}

\textbf{Results.}
Figure~\ref{fig:constrained-wins-validation-frozen} reports how often each method achieves the lowest test MSE among those satisfying a certain DPVar budget. ITD has the highest mean win rate at all three budgets. FBO ranks second under the Tight and Loose budgets, whereas GDP\(^{\ast}\) narrowly surpasses FBO under the Mid budget. This performance is obtained despite FBO and ITD using only \texttt{OUT} labels for accuracy optimization, whereas the baselines use pooled \texttt{IN}+\texttt{OUT} labels. Overall, the two bilevel methods achieve the strongest win rates, particularly under the Tight constraint, although the seed-level results show variability.

\subsection{Observed demographic attributes on Communities \& Crime}
\label{sec:communities-crime}

\textbf{Task and protocol.}
We complement the constructed-attribute benchmark with the Communities \& Crime regression dataset, whose 1,994 rows represent US communities \citep{redmond2002communities}. The outcome is standardized reported violent crime per capita. The sensitive attribute $A$ is the observed ten-dimensional continuous demographic-composition profile: four race/ethnicity percentages, four age-composition percentages, the foreign-born percentage, and the percentage reporting limited English proficiency. We remove these coordinates and directly race-/immigration-conditioned fields from $X$, leaving 72 predictors. Across five split/model seeds, each split has 600 \texttt{IN}, 600 \texttt{OUT}, 397 \texttt{VAL}, and 397 \texttt{TEST} communities; DPVar evaluation uses five cross-fitting folds and three shared restarts. GDP\(^{\ast}\)'s train-only two-component projection explains 63.7\% of $A$'s variance on average, but all methods are evaluated on the original ten-dimensional profile. Full protocol details are in \cref{app:communities-crime}.

\textbf{Results.}
Figure~\ref{fig:communities-crime-frontier} gives the test pareto front frontier after validation selection. FBO, ITD and HSIC trace the lower envelope over most of the DPVar range. The $R^2$ penalty reaches the smallest absolute DPVar at a much larger accuracy cost; GDP\(^{\ast}\) is competitive near the unconstrained end but does not attain the tight range. At very loose budgets, the methods converge to similar MSE. Validation-frozen paired comparisons with GDP\(^{\ast}\), uncertainty, additional fairness diagnostics, and runtime are deferred to \cref{app:communities-crime}.

\begin{figure*}[ht]
  \centering
  \includegraphics[width=0.78\textwidth]{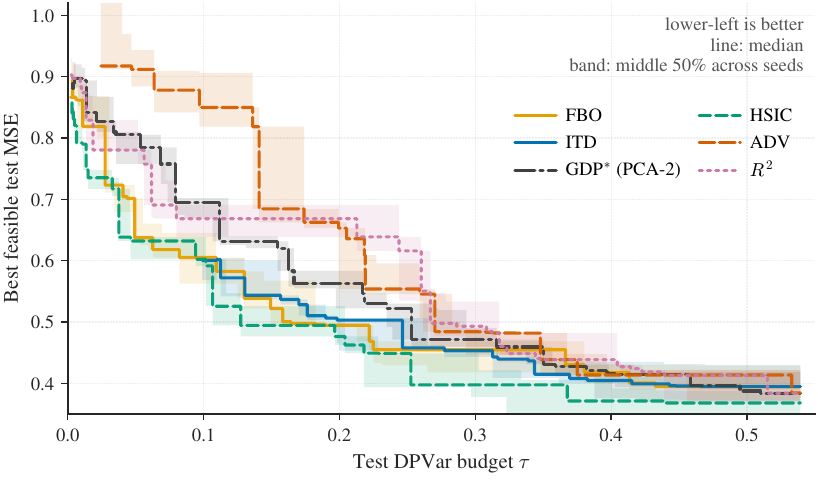}
  \caption{\textbf{Test accuracy--DPVar pareto front on Communities \& Crime data with the ten-dimensional demographic profile.} Lines show the median best-feasible test MSE across five seeds; shaded bands show the 25th–75th percentile range. Lower-left is better.}
  \label{fig:communities-crime-frontier}
\end{figure*}

%% file: ccl.tex
We studied how to train nonlinear regression models so that their average prediction changes less with a continuous, vector-valued sensitive attribute. We measure this change with DPVar. DPVar is a fairness metric that is difficult to compute, since it requires estimating an unknown conditional mean that changes with the predictor during training. We handle this difficulty by specializing the functional bilevel framework of \citet{petrulionyte2024functional}. For squared loss, the resulting population gradient has a closed form, which leads to the FBO update without a Hessian solve. We also study ITD, which differentiates through a finite sequence of updates to the conditional-mean model.

To test the methods across many datasets, we built a semi-synthetic benchmark from a fixed list of public regression tasks. The rows used to construct each task are separate from those used to train and evaluate the fairness methods. We estimate DPVar by fitting the conditional-mean model on one part of an evaluation split and measuring it on another, then exchanging the two parts and averaging the results. We further evaluated an observed ten-dimensional demographic-composition attribute on Communities \& Crime, where FBO and ITD trace competitive accuracy--mean-parity trade-offs across five split/model seeds.

\textbf{Scope and limitations.}
DPVar is a mean-parity criterion and should not be interpreted as a surrogate for full demographic parity in applications where distributional equality is required: it can leave higher-order conditional disparities -- variance, tails, or richer distributional discrepancies -- uncontrolled. Both training and evaluation also require estimating a conditional-mean function of $A$; when $A$ is high-dimensional, this estimation problem can itself be statistically challenging. Finally, our tabular benchmark is semi-synthetic: the selected sensitive coordinates are used to create controlled tasks, not to make application-specific fairness claims about the original datasets. Communities \& Crime supplies observed demographic composition, but its units are communities, and its outcome is reported crime; our study is therefore a contextual methodological audit, not an individual-level or causal fairness claim.

\textbf{Broader impact.}
This work develops methodology for controlling mean disparities in predictions across continuous, high-dimensional sensitive attributes, with potential applications in domains such as credit scoring, insurance pricing, hiring, and recommendation, where average disparities across demographic profiles are a primary fairness concern. By making such control compatible with neural-network predictors and with sensitive attributes that are themselves continuous or vector-valued, the framework can reduce reliance on coarse demographic groupings that may misrepresent affected populations. At the same time, DPVar is a mean-parity criterion: it does not guarantee distributional fairness, equality of opportunity, or individual fairness, and a low DPVar value should not be interpreted as a certificate of overall fairness, which admits context-dependent definitions.

%% file: algorithms.tex
\begin{algorithm}[!ht]
\caption{EstimateConditionalMean}
\label{alg:inner-fit}
\begin{algorithmic}[1]
\State \textbf{Input:} predictor parameters $\omega$; initial inner parameters $\phi_0$; inner stepsize $\eta_{\mathrm{in}}$; number of inner steps $M$; \texttt{IN} split
\State $\phi \leftarrow \phi_0$
\For{$m=1,\dots,M$}
    \State Compute the empirical inner objective
    \[
    \hat L_{\mathrm{in}}(\omega,\phi)
    =
    \frac{1}{|\texttt{IN}|}
    \sum_{(x_i,a_i)\in \texttt{IN}}
    \big(f_\omega(x_i)-h_\phi(a_i)\big)^2
    \]
    \State Update
    \[
    \phi \leftarrow \phi - \eta_{\mathrm{in}} \nabla_\phi \hat L_{\mathrm{in}}(\omega,\phi)
    \]
\EndFor
\State \textbf{Return:} $\phi$
\end{algorithmic}
\end{algorithm}

\begin{algorithm}[!ht]
\caption{ComputeHypergradient}
\label{alg:hypergradient}
\begin{algorithmic}[1]
\State \textbf{Input:} predictor parameters $\omega$; inner parameters $\phi$; fairness weight $\alpha$; hypergradient mode \textsc{HG} $\in \{\textsc{FBO},\textsc{ITD}\}$; unrolling length $K$; inner stepsize $\eta_{\mathrm{in}}$; \texttt{IN} split; \texttt{OUT} split
\If{\textsc{HG} = \textsc{FBO}}
    \State Compute
    \[
    \hat\mu \leftarrow \frac{1}{|\texttt{OUT}|}\sum_{a_i\in\texttt{OUT}} h_\phi(a_i)
    \]
    \State Treat $h_\phi(a_i)$ and $\hat\mu$ as fixed when differentiating with respect to $\omega$
    \State Compute the accuracy gradient
    \[
    g_{\mathrm{acc}}
    \leftarrow
    \nabla_\omega
    \left[
    \frac{1}{|\texttt{OUT}|}
    \sum_{(x_i,y_i)\in\texttt{OUT}}
    \big(f_\omega(x_i)-y_i\big)^2
    \right]
    \]
    \State Compute the fairness gradient
    \[
    g_{\mathrm{fair}}
    \leftarrow
    \nabla_\omega
    \left[
    \frac{2\alpha}{|\texttt{IN}|}
    \sum_{(x_i,a_i)\in\texttt{IN}}
    \big(h_\phi(a_i)-\hat\mu\big) f_\omega(x_i)
    \right]
    \]
    \State \textbf{Return:} $g_{\mathrm{acc}} + g_{\mathrm{fair}}$
\Else
    \State $\tilde\phi^{(0)} \leftarrow \phi$
    \For{$k=0,\dots,K-1$}
        \State $\tilde\phi^{(k+1)} \leftarrow \tilde\phi^{(k)} - \eta_{\mathrm{in}}
        \nabla_\phi \hat L_{\mathrm{in}}(\omega,\tilde\phi^{(k)})$
    \EndFor
    \State Define the unrolled empirical outer objective
    \[
    \hat L_{\mathrm{out}}^{\mathrm{ITD}}(\omega)
    =
    \frac{1}{|\texttt{OUT}|}
    \sum_{(x_i,y_i)\in\texttt{OUT}}
    \big(f_\omega(x_i)-y_i\big)^2
    +
    \alpha\,
    \widehat{\mathrm{Var}}_{a_i\in\texttt{OUT}}
    \big(h_{\tilde\phi^{(K)}}(a_i)\big)
    \]
    \State \textbf{Return:} $\nabla_\omega \hat L_{\mathrm{out}}^{\mathrm{ITD}}(\omega)$
\EndIf
\end{algorithmic}
\end{algorithm}